%% file: paper.tex
\documentclass[submission]{eptcs}
\providecommand{\event}{ACT 2024}
\usepackage{amsmath,amsfonts,amssymb,amsthm,newtxmath, wrapfig}

\usepackage[english]{babel}
\usepackage{braket}
\usepackage{hyperref}
\usepackage{breakurl}
\usepackage{tikzit, tikz-cd}
\usetikzlibrary{decorations.markings}
\input{cybercatrl.tikzstyles}

\newtheorem{proposition}{Proposition}[section]%
\theoremstyle{definition}
\newtheorem{example}{Example}[section]%
\newtheorem{remark}{Remark}[section]%

\newcommand{\sctikzfig}[2][.8]{\begin{center}\scalebox{#1}{\input{#2.tikz}}\end{center}}

\renewcommand\Set{\mathbf{Set}}

\newcommand\Cat{\mathbf{Cat}}
\newcommand\Para{\mathbf{Para}}
\newcommand\CoPara{\mathbf{CoPara}}
\newcommand\Optic{\mathbf{Optic}}
\newcommand\Lens{\mathbf{Lens}}
\newcommand\op{\mathrm{op}}
\newcommand\id{\mathrm{id}}

\newcommand\K{\mathbb K} 
\newcommand\V{\mathbb V} 
\newcommand\I{\mathbb I} 
\newcommand\B{\mathfrak B} 
\newcommand\Bval{\B_\mathrm{val}}
\newcommand\Bpol{\B_\mathrm{pol}}
\newcommand\pibeh{\pi_\mathrm{beh}}
\newcommand\pitgt{\pi_\mathrm{tgt}}
\newcommand\ret{\Upsilon} 

\newcommand\M{\mathcal M}
\newcommand\C{\mathcal C}
\newcommand\D{\mathcal D}

\newcommand\R{\mathbb R} 

\DeclareMathOperator{\argmax}{\mathrm{argmax}}

\title{Reinforcement Learning in \\ Categorical Cybernetics}
\author{Jules Hedges \and Riu Rodr\'iguez Sakamoto}
\def\titlerunning{Reinforcement Learning in Categorical Cybernetics}
\def\authorrunning{J. Hedges \& R. Rodr\'iguez Sakamoto}

\begin{document}

\maketitle

\begin{abstract}
We show that several major algorithms of reinforcement learning (RL) fit into the framework of categorical cybernetics, that is to say, parametrised bidirectional processes. We build on our previous work in which we show that value iteration can be represented by precomposition with a certain optic. The outline of the main construction in this paper is: (1) We extend the Bellman operators to parametrised optics that apply to action-value functions and depend on a sample. (2) We apply a representable contravariant functor, obtaining a parametrised function that applies the Bellman iteration. (3) This parametrised function becomes the backward pass of another parametrised optic that represents the model, which interacts with an environment via an agent. Thus, parametrised optics appear in two different ways in our construction, with one becoming part of the other. As we show, many of the major classes of algorithms in RL can be seen as different extremal cases of this general setup: dynamic programming, Monte Carlo methods, temporal difference learning, and deep RL. We see this as strong evidence that this approach is a natural one and believe that it will be a fruitful way to think about RL in the future.
\end{abstract}

\section{Introduction}

\emph{Reinforcement learning} (RL) refers to a class of methods in machine learning for optimising a long-run reward during interaction with an unknown environment. It is considered one of the major pillars of machine learning, along with \emph{deep learning} (neural networks and differentiable programming), \emph{unsupervised learning} (statistical clustering methods, which includes topological data analysis \cite{ghrist-persistent-topology}) and \emph{variational learning} (Bayesian inference and related probabilistic methods). It can be seen as an extension of \emph{dynamic programming} methods in optimal control theory \cite{bertsekas_book}, which drops the assumption that a model of the environment is known. RL, combined with deep learning methods to produce \emph{deep RL}, notably achieved state of the art success in practical game playing, with AlphaGo \cite{alphago} defeating the human Go champion in 2016 and AlphaStar \cite{alphastar} achieving Grandmaster status in the real time strategy game StarCraft II.

In this paper we show that several major algorithms of reinforcement learning fit into the framework of \emph{categorical cybernetics}, that is to say, \emph{parametrised bidirectional processes} \cite{towards-foundations}. This branch of applied category theory has already been applied to deep learning \cite{foundations-gradient-learning,bruno_thesis}, variational learning \cite{compositional-bayesian,toby-thesis} and game theory \cite{compositional-game-theory,bayesian_open_games}. It is also a close relative of the \emph{categorical systems theory} of Myers, Spivak and others \cite{david-jaz-book,poly_book}. 

We build on our previous work \cite{value-iteration-optic-composition} in which we show that \emph{value iteration}, a fundamental method common to both dynamic programming and RL, can be represented (in the technical sense) by precomposition with a certain optic. Specifically, for each \emph{policy} $\pi$ we define an \emph{optic} $\mathbb B (\pi) : \binom{S}{\R} \to \binom{S}{\R}$, where $S$ is the set of states of a Markov decision process. This has the property that for any \emph{value function} $V : S \to \R$, represented as an optic $V : \binom{S}{\R} \to I$, $V \circ \mathbb B (\pi)$ is a better value function. This precomposition with $\mathbb B (\pi)$ is called a \emph{Bellman operator}.
This bidirectional approach differs from Bakirtzis, Savvas and Topcu's categorical specification of MDPs \cite{bakirtzis_categorical_semantics_rl}, which focuses on several compositional aspects such as subprocesses and sequential tasks. Our subject matter is the structure of the algorithms that are used in such RL environments and how they relate to each other.

The outline of the main construction in this paper is: (1) We extend $\mathbb B$ to a \emph{parametrised optic} representing a more general class of Bellman operators that apply to \emph{action-value functions} and depend on a \emph{sample} as a parameter. (2) We apply $\K$, a representable contravariant functor that already plays a foundational role in compositional game theory, obtaining a parametrised function $\B = \K (\mathbb B)$ that applies the Bellman iteration. (3) This parametrised function becomes the backward pass of another parametrised optic that represents the \emph{model}, which interacts with an \emph{environment} via an \emph{agent}. Thus, parametrised optics appear in two different ways in our construction, with one becoming part of the other. This stays within the existing ingredients of categorical cybernetics, but combines them in a way that has not been seen elsewhere.

As we show, many of the major classes of algorithms in RL can be seen as extremal cases of this general setup: dynamic programming, Monte Carlo methods, temporal difference learning, and deep RL. We see this as strong evidence that this approach is a natural one and believe that it will be a fruitful way to think about RL in the future.
For now our goal is merely to achieve a better conceptual understanding of RL, although the hope is that this will eventually translate into quantifiable benefits such as improved modelling techniques.
Although we focus on single-agent RL, the compositionality of our methods makes them naturally well-suited to \emph{multi-agent} RL, which is a close relative of game theory.

\section{Background: Reinforcement learning}

Algorithms in RL specify how agents learn optimal behaviours through interaction with their environment.
This interaction provides feedback to actions, and is the key feature that differentiates it with respect to supervised and unsupervised learning.
The fundamental goal of RL is to enable agents to make sequential decisions in dynamic environments to maximize long-term cumulative rewards.
This process involves the agent taking actions, observing the resulting states and rewards, and using this information to update its decision-making strategy over time.

Our approach to study these algorithms is structural, and the main structural distinction is between the agent and the environment.
The \textbf{environment} represents the external system with which the agent interacts, and is assumed be a Markov decision process.
To quickly recall, a Markov process (MP) consists of a set of states $S$ and a stochastic transition function $t : S \to DS$, where $D$ is some probability monad over $\Set$.
A Markov reward process (MRP) is a Markov process with an additional function $r:S\to D \R$ that outputs the immediate \emph{reward} for the current state.
This reward function can be in general be correlated with the transition, in which case we write $t : S \to D(S\times \R)$.
When clear from context, we will abuse the notation and write $r$ for the reward of a particular state, e.g. $(s',r)\sim t(s)$.
A Markov decision process (MDP) is a MRP with a set $A$ of actions, whose transition and reward functions now depend on the action taken at each state too, and is also in general correlated too $t:S\times A\to D(S\times\R)$.\footnote{The cases where the reward function is decorrelated with $t$ as in $S\times A\to D(\R)$, $S\to D(\R)$, or $S\times A\times S\to D(\R)$ can be embedded in our modelling choice for $t$.}
An agent's goal is to maximize the \emph{expected long-run reward} $\sum \gamma^i r(s_i)$, where $0 < \gamma \leq 1$ is a hyperparameter called the \textbf{discount factor} which controls the agent's ``patience'', or preference between rewards in the present and rewards in the future.

The environment's response to an agent's action is given by transition dynamics that can be assumed to have the Markov property, and the environment's state is known to the agent. 
When the agent only has access to a partial observation of the state, we speak of a partially observable MDP (POMDP).

The \textbf{agent} has as core components the policy, the reward, the value function and the internal model.
A \textbf{policy} or scheduler $\pi:S\to TA$ defines its strategy mapping states to actions.
The policy is either single-valued or deterministic ($T=1$ the identity functor), many-valued ($T=\mathcal{P}$ the powerset functor, like $\argmax$) or probabilistic ($T=D$ the distribution functor, like $\varepsilon$-greedy), with probabilistic being the most common.
The \textbf{reward} is the immediate response of the environment after an action, and the maximization of its expected cumulative sum is the goal of the agent under the \emph{reward hypothesis} \cite{reward_is_enough}.
A \textbf{value function} estimates this expected long-term reward associated with following a particular policy. Usually one works with either a state value function $V:S\to \R$ or a state-action value function $Q:S\times A\to \R$, where $V(s)$ estimates the long-run reward of following a certain policy from each state, and $Q(s, a)$ estimates the long-run reward of taking each action in each state and then following a certain policy after that.
Both policies and value functions are characterised as solutions of functional equations known as \textbf{Bellman equations}, using the temporal `self-similarity' of MDPs.

When $S$ and $A$ are finite sets the function $Q$ is typically implemented as a mutable lookup table called a Q-table or Q-matrix.
A \textbf{model} is an approximation or representation of the environment's dynamics, allowing the agent to simulate or predict future states and rewards.
One surrogate objective of an agent is to improve its model.
Not all agents have models, so there's a distinction between \textbf{model-based} and \textbf{model-free} methods.

Methods whose policies for environment interaction $\pibeh$ (``behaviour policy'') are different to the ones for model improvement $\pitgt$ (``target policy'') are called \textbf{off-policy}. \textbf{On-policy} methods only have a single policy.
Finally, another distinction is drawn between \textbf{online} and \textbf{offline} or \textbf{batch RL} methods, where the former family learns while interacting with the environment, while the latter learns from pre-recorded experiences.

RL encompasses many algorithms and methodologies, including dynamic programming, Monte Carlo methods, temporal difference learning, deep reinforcement learning, and more.
This diversity of methods employs experimental and formal justifications to tackle weak spots in this learning theory such as the credit-assignment problem, the exploration-exploitation tradeoff and coping with state that is hidden or too big to represent explicitly. Many of these were problems already identified in preceding fields such as psychology and neuroscience \cite{kaelbling_survey}.

\subsection{Dynamic programming}

Dynamic programming (DP) methods are an idealized class of model-based algorithms that do not need to interact with the environment because they have a perfect model of it as an MDP.
They are not usually used in their classical formulation that we describe next in practical settings because of the perfect model assumption and their high computational expense, and serve rather as a theoretical baseline to approximation methods and other RL techniques.

The idea behind DP is to treat the Bellman equation for the optimal value of a policy and the Bellman equation for the optimal policy of a value function as update operators on a space of value functions.

The search for an optimal policy happens entirely within the agent's model, interleaving two feedback operations called the \emph{value improvement} and \emph{policy improvement} steps which treat the Bellman equations as update rules. This process of updating previous estimates is called \textbf{bootstrapping}.
\begin{itemize}
    \item \textbf{Value improvement} or policy evaluation updates the value function $V:S\to \R$ pointwise by traversing the state space $S$ and updating the state's estimated value $V(s)$ with the expected discounted value after one simulation step:
    \begin{equation} \label{eq:VIP}
        V(s) \gets \mathbb{E}_{\mkern-14mu\substack{a\sim \pi(s)\\ (s',r)\sim t(s,a)}}[r+\gamma V(s')] = \sum_{a\in A}\pi(a\mid s)\sum_{\substack{s'\in S\\ r\in \R}}t(s',r\mid s,a) \cdot(r + \gamma V(s'))
	\end{equation}
    Here we are interchangeably considering a stochastic function $f : X \to DY$ as a function $f : Y \times X \to [0,1]$, whose causality is still reflected by the traditional middle bar $f(y\mid x)$.
	The sum over $\R$ makes sense when $D$ is finite support probability distributions, and in more general settings is replaced with an integral.
    We write $\Bval(V,\pi)(s)=\mathbb{E}[r+\gamma V(s')]$ for the operator $\Bval:\R^S\times (TA)^S\to \R^S$.
	We will discuss Bellman operators in Section~\ref{sec:Bellman_operators}.
    \item \textbf{Policy improvement} updates the policy function $\pi:S\to TA$ pointwise by traversing the state space $S$ and updating the action taken in the state $\pi(s)$ with $\argmax_a\mathbb{E}_{(s',r)\sim t(s,a)}[r+\gamma V(s')]$.  
    Similarly, we write $\Bpol(V)(s)=\argmax_a\mathbb{E}_{s',r\sim t(s,a)}[r+\gamma V(s')]$ for the operator $\Bpol:\R^S\to (TA)^S$.
\end{itemize}

Depending on the sequencing of these two steps, we have three classic algorithms, where we write $(-)^\dagger$ for the (in practice approximate) fixpoint of the operator and, respectively, $\overline{\Bpol}(V,\pi)=(V,\Bpol(V))$ and $\overline{\Bval}(V,\pi)=(\Bval(V,\pi),\pi)$ for the embeddings of the two Bellman operators as maps $\R^S\times (TA)^S\to \R^S\times (TA)^S$: \textbf{policy iteration} (PIT) as $(\overline{\Bpol}\circ\overline{\Bval}^\dagger)^\dagger$, \textbf{value iteration} (VIT) as $(\overline{\Bpol}\circ\overline{\Bval})^\dagger$ and \textbf{generalized policy iteration} (GPI) as $(\overline{\Bpol}^m\circ\overline{\Bval}^n)^\dagger$ for $m,n > 0$.


\subsection{Monte Carlo}

Monte Carlo (MC) methods are antithetical to DP, because they don't assume any prior knowledge of the environment's dynamics.
Without this knowledge, the way to learn the value function and obtain a optimal policy is to estimate it from sample trajectories.
Averaging over many trajectories should converge to the expected value.

The agent's internal model consists of a value function $Q:S\times A\to \R$, from which a policy like the $\varepsilon$-greedy $\pi:S\to DA$ is derived:
$\pi(s) = \argmax_a Q(s,a)$ with probability $1-\varepsilon$ and uniformly random between all actions with probability $\varepsilon$.
The value function improvement is pointwise, but unlike DP, MC improves $Q(s,a)$ by averaging over many returns that start at $(s,a)$.
Given a single episode $(s,a,r,s',a',r',\dots)$ starting at $(s,a)$, the update \textbf{target} becomes $G=\sum_t \gamma^t r_t$, and the value function updates as
\begin{equation} \label{eq:model_update}
    Q(s,a) = (1-\alpha)Q(s,a) + \alpha G
\end{equation}
where the learning rate $\alpha:[0,1]$ is a step size hyperparameter.
Note that the lack of bootstrapping is shown by the fact that $G$ does not contain any reference to the value function.

\subsection{Temporal difference learning}

Temporal difference learning (TD) is a class of methods that learn from both the interaction with the environment (MC's sampling) and from previous estimates of the value function (DP's bootstrapping).

Given a finite episode $(s,a,r,\dots,s_n,a_n)$ starting at $(s,a)$, we can modify the target for \eqref{eq:model_update} to consist of the discounted sum of the $n-1$ returns and an estimated long-run value of the last state-action pair. We write $n$-TD for the class of TD methods whose trajectories contain $n$ return values.

\begin{example}[SARSA \cite{sarsa}]
    SARSA is a $1$-TD on-policy control method, which updates the $(s,a)$-indexed Q-value with the target $G = r + \gamma Q(s',a')$. The name originates from the model feedback consisting of a $1$-step episode $(s,a,r,s',a')$.
    Some variants of SARSA include $n$-SARSA, with $G = \sum_{t=1}^{n-1}\gamma^t r_t + \gamma^nQ(s_n,a_n)$, and Exp-SARSA, with $G = r + \gamma\mathbb{E}_{a\sim\pitgt(s)}Q(s,a)$, which is off-policy because the last action is determined in expectation by a target policy $\pitgt$.
\end{example}

\begin{example}[Q-learning \cite{q_learning}]
	In Q-learning, given the current state $s$ the agent performs an action $a \sim \pibeh (s)$ using a policy derived from its internal Q-table, for example an $\varepsilon$-greedy policy, and gets from the environment the reward $r$ and the next state $s'$.
    The feedback to the model is the tuple $(s,a,r,s')$.
    The model then updates its Q-table with its target policy $G = r+\gamma Q(s,\pitgt(s))=r+\gamma \max_{a'\in A}Q(s',a')$.
    It is an off-policy method because the last action used to compute the update is $\pitgt(s') = \argmax_{a'\in A}Q(s',a')$ and not $\pibeh (s')$.
\end{example}

Q-learning is the first appearance of a major subtlety of RL: the distinction between actions that the agent actually performs during an interaction with its environment, and actions which are ``internal'' or ``simulated''. The actions that the agent actually performs in Q-learning are always drawn from the policy $\pibeh$, whereas the action $\argmax_{a' \in A} Q (s', a')$ is used only when computing updates. We can consider this to be a separate target policy, $\pitgt (s') = \argmax_{a' \in A} Q (s', a')$.

\section{Background: categorical cybernetics}

In this section we quickly recall the main ideas of categorical cybernetics, mostly from \cite{towards-foundations}. 

\subsection{Actegories}

Given a monoidal category $\M$ and a category $\C$, an action of $\M$ on $\C$, also called an \textbf{actegory}, is a functor $\bullet : \M \times \C \to \C$ together with coherent isomorphisms $I \bullet X \cong X$ and $(M \otimes N) \bullet X \cong M \bullet (N \bullet X)$ \cite{actegories}. Every monoidal category has a self-action given by $\otimes : \C \times \C \to \C$.

If $\M$ and $\C$ are monoidal categories and $F : \M \to \C$ is a strong monoidal functor, then $M \bullet X = F (M) \otimes X$ is an actegory. A coherent action of one symmetric monoidal category on another, called a \textbf{symmetric actegory}, is necessarily of this form. For example, the self-action of a symmetric monoidal category is a symmetric actegory given by the identity functor $\C \to \C$. All actegories in this paper will be symmetric.

An \textbf{enrichment} of a category $\C$ in a monoidal category $\M$ is a functor $[-,-] : \C^\op \times \C \to \M$ plus additional data and conditions. There is a tight connection between actegories and enrichments: if $\bullet$ is any actegory such that every $- \bullet X : \M \to \C$ has a right adjoint $[X, -] : \C \to \M$ (called a closed actegory) then $[-, -]$ is an enrichment, and conversely if $[-, -]$ is an enrichment such that every $[X, -]$ has a left adjoint $- \bullet X$ (called a copowered or tensored enrichment) then $\bullet$ is an action. For example, if $\C$ is any category with all coproducts then it has a tensored enrichment in the cartesian monoidal category $\Set$ and therefore an action $\bullet : \Set \times \C \to \C$ given by $M \bullet X = \sum_M X$.

\subsection{Parametrisation}

Given an actegory $\bullet : \M \times \C \to \C$, a \textbf{parametrised morphism} $f : X \to Y$ in $\C$ is a pair of an object $M : \M$ and a morphism $f : M \bullet X \to Y$. The identity parametrised morphism is given by $I : \M$ and $\id_X : I \bullet X \cong X \to X$. The composite of $(M, f : M \bullet X \to Y)$ and $(N, g : N \bullet Y \to Z)$ has parameter $N \otimes M$ and morphism $(N \otimes M) \bullet X \overset\cong\longrightarrow N \bullet (M \bullet X) \xrightarrow{N \bullet f} N \bullet Y \overset{g}\longrightarrow Z$. A reparametrisation from $(M, f : M \bullet X \to Y)$ to $(N, g : N \bullet X \to Y)$ is a morphism $h : M \to N$ in $\C$ such that $f = g \circ (h \bullet X)$.

Given an actegory $\bullet : \M \times \C \to \C$, we have a bicategory whose objects are objects of $\C$, 1-cells are parametrised morphisms and 2-cells are reparametrisations. This may be referred to by $\Para_\M (\C)$, $\Para_\bullet (\C)$ or simply $\Para (\C)$ when unambiguous. We believe that when $\bullet$ is a symmetric monoidal actegory, $\Para_\bullet (\C)$ is a symmetric monoidal bicategory \cite{constructing_monoidal_bicategories} but this has not yet been proven.

When we have an action $\bullet : \M \times \C \to \C$ and a symmetric lax monoidal functor $W : \M \to \Set$ (or sometimes $\Cat$) with its cartesian product, and we extend $\bullet$ to an action of the category of elements $\int W$ by precomposing with the discrete fibration $\pi : \int W \to \M$ to obtain $(M, w) \bullet X = M \bullet X$. When $W$ is lax monoidal with laxator $\nabla : W (M) \times W(N) \to W (M \otimes N)$, $\int W$ gains a symmetric monoidal product $(M, w_M) \otimes (N, w_N) = (M \otimes N, w_M \nabla w_N)$. We write $\Para^W_\M (\C)$ for $\Para_{\int W} (\C)$, and call this \textbf{weighted parametrisation} \cite{bruno_thesis}.

Dually, a \textbf{coparametrised morphism} $f : X \to Y$ is a pair of an object $M : \M$ and a morphism $f : X \to M \bullet Y$. There is a category $\CoPara_\bullet (\C)$ of objects, coparametrised morphisms and reparametrisations.

Given a category $\C$ enriched in a monoidal category $\M$, an \textbf{externally parametrised morphism} $f : X \to Y$ of $\C$ is a pair of an object $M$ of $\M$ and a morphism $f : M \to [X, Y]$ of $\M$ \cite{toby-thesis}. There is once again a bicategory $\Para_\M (\C)$ of externally parametrised morphisms. In the case of a tensored enrichment this bicategory is equivalent to the previous one, but there are also interesting cases when they differ. Coparametrised morphisms cannot be defined for an enrichment that is not tensored.

\subsection{Optics}

\begin{wrapfigure}{r}{0.23\textwidth}
   \sctikzfig{figs/optic_composition}
	\caption{Alternative notations for optic composition}
	\label{fig:optic_composition}
\end{wrapfigure}

Given a monoidal category $\M$ acting on categories $\C$ and $\D$, and given objects $X, Y$ of $\C$ and $X', Y'$ of $\D$, a \textbf{mixed optic} $\binom{X}{X'} \to \binom{Y}{Y'}$ is an equivalence class of triples of an object $M : \M$, a coparametrised morphism $f : X \to M \bullet Y$ of $\C$ and a parametrised morphism $f' : M \bullet Y' \to X'$ of $\D$. The equivalence classes are generated by reparametrisations and satisfy the universal property of a coend, $\int^{M : \M} \C (X, M \bullet Y) \times \D (M \bullet Y', X')$. There are two different string diagram notations for an optic (figure~\ref{fig:optic_composition}). The first considers them as morphisms of a monoidal category, composing left-to-right, with causality flowing clockwise from top-left. The second considers them as colours of an operad, composing outside-in, with causality flowing left-to-right.

There is a category $\Optic_\M (\C, \D)$ whose objects are pairs and whose morphisms are optics. When both actions are symmetric, or equivalently are defined by a span of symmetric monoidal functors $\C \leftarrow \M \to \D$, then $\Optic_\M (\C, \D)$ is a symmetric monoidal category, with the tensor product on objects being pairwise. 

In the common case that $\M = \C = \D$ acts on itself by monoidal product, we write $\Optic (\C)$. The tensor product of $\Optic (\C)$ is pairwise monoidal product. When the monoidal unit of $\C$ is terminal (which includes all Markov categories) then we have natural isomorphisms $\Optic (\C) \left( I, \binom{X}{X'} \right) \cong \C (I, X)$ and $\Optic (\C) \left( \binom{X}{X'}, I \right) \cong \C (X, X')$. We call morphisms in latter case \textbf{continuations}, and define the representable functor $\K = \Optic (\C) (-, I) : \Optic (\C)^\op \to \Set$.

There are two common cases when the coend in the definition of optics can be eliminated using the ninja Yoneda lemma \cite{riley_optics,fosco_book}. Firstly, when $\M = \C = \D$ acts on itself by cartesian product then there is a natural isomorphism $\Optic \left( \binom{X}{X'}, \binom{Y}{Y'} \right) \cong \C (X, Y) \times \C (X \times Y', X')$. This is usually known as a \textbf{Lens}. Although this case is much easier to understand, there are significant conceptual advantages to the more general definition \cite{bruno_space_time_tradeoffs}. Secondly, when $\M = \C = \D$ acts on itself by a closed monoidal product then there is a natural isomorphism $\Optic \left( \binom{X}{X'}, \binom{Y}{Y'} \right) \cong \C (X, Y \otimes [Y', X'])$. Both of these cases can be generalised to requiring a condition on only one side.

For the cartesian self-action of $\Set$, $\Optic (\Set)$ coincides with the category of monomial endofunctors (those of the form $F (X) = A \times X^B$) and natural transformations. Any cartesian self-action in a category with finite limits can be generalised to \textbf{dependent lenses} (also known as morphisms of \textbf{containers} \cite{categories_containers}), which in the locally cartesian closed case are equivalent to \textbf{polynomial endofunctors} \cite{poly_book}. Finding the best of both worlds between the monoidal and cartesian cases is known as \textbf{dependent optics} \cite{vertechi_dependent_optics} and is only partially understood. There are reasons to want to use dependent optics in this paper because it is common that the available actions of a reinforcement learning agent depends on the current state of the Markov chain \cite{botta_etal_sequential_decision_problems}, but we only consider the simply-typed case in this paper for simplicity.

\subsection{Parametrised optics}

When a category of optics is symmetric monoidal, it admits a self-action. In \cite{towards-foundations} it was identified that the resulting category $\Para (\Optic)$ of \textbf{parametrised optics} is extremely rich, and provides a general-purpose foundation for the study of controlled processes. The study of this is known as \textbf{categorical cybernetics}, which includes compositional game theory \cite{davidad-dioptics}, deep learning \cite{foundations-gradient-learning}, compositional Bayesian inference \cite{compositional-bayesian} and variational learning \cite{toby-thesis}, and applications in software engineering such as \emph{open servers} \cite{open-servers}.

\section{States, contexts and iteration}

An optic (whether parametrised or not) is a process consisting of a forward pass followed by a backward pass. In many applications, including those in this paper, this process is iterated through repeated interaction with an outside environment. In the case of supervised learning, this could simply be samples drawn from a dataset. In this section we will develop a general theory of iterated optics.

\begin{wrapfigure}{r}{0.23\textwidth}
   \sctikzfig{figs/weighted_para}
	\caption{Morphism in $\Para^W (\C)$ (right) and its equivalence class in $\pi_0^*(\Para^W (\C))$ (left).}
	\label{fig:weighted_para}
\end{wrapfigure}

Let $\C$ be a symmetric monoidal category and $W : \C \to \Set$ a symmetric lax monoidal functor. Consider the bicategory $\Para^W (\C)$ generated by the action of $\int W$ on $\C$ given by $(M, w) \bullet X = M \otimes X$. A parametrised morphism $X \to Y$ of $\C$ weighted by $W$ consists of a morphism $M \otimes X \to Y$ together with an element $w \in W (M)$, as depicted in figure~\ref{fig:weighted_para}(right).

Any bicategory can be turned into a 1-category by change of enrichment basis along the connected components functor $\pi_0 : \Cat \to \Set$. This operation quotients together 1-cells that are related by any 2-cell. ($\pi_0$ is right adjoint to the free functor $\Set \to \Cat$, and it is more common to change basis along the left adjoint, which instead quotients out only invertible 2-cells.) The 1-category $\pi_0^* (\Para^W (\C))$ has morphisms that are equivalence classes identifying all ways of making the cut in figure~\ref{fig:weighted_para}(right). This satisfies an important universal property: it is the symmetric monoidal category that results from freely extending $\C$ with a state $w : I \to X$ for each element $w \in W (X)$, for all objects $X$ \cite{hermida_tennent_monoidal_indeterminates}\footnote{Thanks to Nathan Corbyn for bringing this reference to our attention.}.

Let $\C$ be a symmetric monoidal category. We define a symmetric lax monoidal functor called the \textbf{iteration functor}, ${\I : \Optic (\C) \to \Set}$ \cite{iteration_optics}. On objects, we set
\[ \I \binom{X}{X'} = \int^{M : \C} \C (I, M \otimes X) \times \C (M \otimes X', M \otimes X) \]
Given a representative element $(M, x_0, i) \in \I \binom{X}{X'}$ we call $M$ the \textbf{state space}, $x_0 : I \to M \otimes X$ the \textbf{initial state} and $i : M \otimes X' \to M \otimes X$ the \textbf{iterator}.

Given an optic $f = (N, f, f') : \binom{X}{X'} \to \binom{Y}{Y'}$ in $\Optic (\C)$, we get a function $\I (f) : \I \binom{X}{X'} \to \I \binom{Y}{Y'}$ given by taking $(M, x_0, i)$ to the state space $M \otimes N$, the initial state $I \overset{x_0}\longrightarrow M \otimes X \xrightarrow{M \otimes f} M \otimes N \otimes Y$, and the iterator
\[ M \otimes N \otimes Y' \xrightarrow{M \otimes f'} M \otimes X' \overset{i}\longrightarrow M \otimes X \xrightarrow{M \otimes f} M \otimes N \otimes Y \]
This can be easily checked to be functorial and well-defined (see Appendix \ref{appendix} for proofs).

\begin{proposition}\label{prop:well-defined_iterator}
	The iterator $\I:\Optic(\C)\to\Set$ is well-defined.
\end{proposition}

\begin{proposition}\label{prop:functorial_iterator}
	The iterator $\I:\Optic(\C) \to \Set$ is functorial.
\end{proposition}

When $\C = \Set$ and similar cases, given an element $i = (M, (m_0, x_0), i) \in \I \binom{X}{X'}$ and a function $k : X \to X'$, we can define an infinite sequence $\braket{k | i} : X^\omega$ by the corecursive formula
\[ \braket{k | M, (m_0, x_0), i} = x_0 : \braket{k | M, i (m_0, k (x_0)), i} \]
This defines a dinatural transformation $\braket{- | -} : \mathbb K \binom{X}{X'} \times \I \binom{X}{X'} \to X^\omega$ which is well-defined.

\begin{proposition}\label{prop:well-defined_streams}
	The map $\braket{- \mid -} : \mathbb K \binom{X}{X'} \times \I \binom{X}{X'} \to X^\omega$ is a dinatural transformation when $\C=\Set$.
\end{proposition}

In the general case, we believe this can be accomplished using the machinery of monoidal streams \cite{monoidal-streams}.

We also have an evident laxator $\nabla : \I \binom{X}{X'} \times \I \binom{Y}{Y'} \to \I \binom{X \otimes Y}{X' \otimes Y'}$ defined up to symmetries by
 \[(M, x_0, i) \nabla (M', x_0', i') = (M \otimes M', x_0 \otimes x_0', i_0 \otimes i_0') \]
The resulting symmetric monoidal category $\Optic^\I (\C) = \pi_0^* \left( \Para^\I (\Optic (\C)) \right)$ extends $\Optic (\C)$ with states $I \to \binom{X}{X'}$ defined by elements of $\I \binom{X}{X'}$. A typical morphism is depicted in figure ~\ref{fig:iterated_optic}.

Given a symmetric monoidal category $\C$, a \textbf{comorphism} $X \to X'$, which could also be called a \textbf{context} for morphisms $X \to X'$, is a state $I \to \binom{X}{X'}$ of $\Optic (\C)$. When $\C$ is itself a category of optics, this is known as \textbf{double optics} and is a central idea of Bayesian open games \cite{bayesian_open_games}. These can be depicted as combs with 1 hole and bidirectional wires, or combs with 2 holes and only forwards wires.

Given a symmetric monoidal category $\C$, a functor $\Optic (\C) \to \Set$ can be equivalently defined as a \textbf{Tambara comodule}: a functor $W : \C \times \C^\op \to \Set$ (note the variance) equipped with a natural family of functions $W (M \otimes X, M \otimes Y) \to W (X, Y)$. This is a dualisation of the fundamental theorem of Tambara theory \cite{pastro-street,categorical_update}. Given this data, a generalised comorphism $X \to X'$ can be defined as an element of $W \binom{X}{X'}$, that is, a state of $\Optic^W (\C) = \pi_0^* \left( \Para^W (\Optic (\C)) \right)$. This construction appears in Sec.9 of \cite{game_semantics_game_theory}.\footnote{The first author has been working on the theory of generalised contexts for several years, but has yet to find a compelling application outside of categorical cybernetics.}

Putting this together, we can define an \textbf{iteration context} for optics $\binom{X}{X'} \to \binom{Y}{Y'}$ as a (representable) state of $\Optic (\Optic^\I (\C))$. This defines a functor $\I_{\textrm{env}} : \Optic^2 (\C) \to \Set$ depicted in figure~\ref{fig:Iteration_environments}(top), and by pulling the state variable of $i$ through $k$ we can define it equivalently by a coend over $\C$ rather than over $\Optic^\I (\C)$:
\[ \I_{\textrm{env}} \left( \binom{X}{X'}, \binom{Y}{Y'} \right) \cong \int^{M, M' : \C} \C (I, M \otimes X) \times \C (M \otimes Y, M' \otimes Y') \times \C (M' \otimes X', M \otimes X) \]
This can be equivalently depicted as a 3-hole comb in figure~\ref{fig:Iteration_environments}(bottom), and we can unroll this $n$ steps to produce a $2n$-hole comb, including an $\omega$-comb \cite{monoidal-streams} for the limiting case.

\begin{wrapfigure}{R}{0.25\textwidth}
	\sctikzfig{figs/optic_context}
	\sctikzfig{figs/3_hole_comb}
	\caption{Iteration contexts as states of $\Optic^2(\C)$ (top) and 3-hole combs (bottom).}
	\label{fig:Iteration_environments}
\end{wrapfigure}

\section{Bellman operators}\label{sec:Bellman_operators}

A \emph{Bellman operator} is a self-mapping of a function space of either state-value functions or state-action-value functions, which iteratively improves the estimation of values. For dynamic programming, the most basic Bellman operator $\B_\pi = \Bval (-, \pi) : \R^S \to \R^S$ is defined by
\[ \B_\pi (V) (s) = \mathbb E_{(r, s') \sim t (s, \pi (s))} [r + \gamma V (s')] \]
The functional equation $V = \B_\pi (V)$ is called a \emph{Bellman equation}, and its solution $V$ characterises the long-run values of the policy $\pi$. Provided $S$ is finite and $0 < \gamma < 1$, $\B_\pi$ is a contraction mapping on the supremum metric of $\R^S$, and therefore iterating $\B_\pi$ from any initial estimate of $V$ will converge to the unique solution of the Bellman equation.

In \cite{value-iteration-optic-composition} we showed that $\B_\pi$ has the form $\B_\pi = \K (\ell_\pi)$, where $\ell_\pi = \binom{\ell_f}{\ell_b} : \binom{S}{\R} \to \binom{S}{\R}$ is the mixed optic in the category $\Optic_{\mathrm{Kl} (D)}(\mathrm{Kl} (D), \mathrm{EM} (D))$ defined by $\ell_f: S \to D(S\times \R)$ which maps $s \mapsto t(s,\pi(s))$ and $\ell_b: D(\R) \times \R \to \R$ which maps $r, v \mapsto \mathbb E [r] + \gamma v$.
Here $D$ is the finite support probability monad on $\Set$, whose Eilenberg-Moore category is convex sets \cite{fritz_convex_spaces}, with $\mathrm{Kl} (D)$ acting on $\mathrm{EM} (D)$ via the embedding of free algebras as algebras $D : \mathrm{Kl} (D) \hookrightarrow \mathrm{EM} (D)$, which is a strong monoidal functor.

That is to say, if we represent a value function $V : S \to \R$ by a costate of optics $V : \binom{S}{\R} \to \binom{1}{1}$, then the costate $V \circ \mathcal B_\pi : \binom{S}{\R} \to \binom{1}{1}$ similarly represents $\B (V) : S \to \R$.
This is a refinement of the usual view of Bellman operators as endomorphisms of function spaces.

Strictly speaking it would be preferable to use a category of optics whose forward objects are finite sets and backward objects are complete metric spaces, using a suitable probability monad such as Kantorovich \cite{fritz_perrone_probability_monad_colimit}, in order to guarantee convergence in all cases. However we leave the details of this for future work.
For the general theory of RL one can work with Markov categories \cite{Synthetic_approach}, and especially representable Markov categories \cite{representable_markov_categories} to handle the interplay between \emph{distributions} and \emph{samples}.

This type of Bellman operator updates an entire value function at once, as is most common in basic dynamic programming. However in reinforcement learning it is far more common to update a Q-matrix one entry at a time, with a sample determining which state-action pair is to have its value updated. Bellman operators of that form do not directly factor through $\K$. However we can fix this with a small change in perspective: we consider Bellman operators that return a \emph{delta} or \emph{change} to a Q-matrix, apply $\K$ to that, and then apply the change to obtain a new Q-matrix in the world of functions rather than optics. This same change of perspective is required anyway to describe deep RL where the delta is replaced with a cotangent vector, so it is an interesting observation that the category theory suggests the same for the more discrete setting of tabular RL.

One may ask whether $\Bpol$ arises also as the image under $\K$ of a some optic, and the answer is negative.
This was an ambiguous point in \cite{value-iteration-optic-composition} where the policy improvement step could not be stated as an optic, even though both $\Bval$ and $\Bpol$ are treated as similar contraction operators in the dynamic programming literature. We can now give a more explicit answer; $\Bpol$ involves two currying maps ($\lambda$), so it is not in the image of an optic under the continuation functor (see figure \ref{fig:VIP_PIP_in_Set}).

\begin{figure}[ht]
	\begin{minipage}{0.5\textwidth}
        \sctikzfig{figs/VIP_in_LSet}
	\end{minipage}\begin{minipage}{0.5\textwidth}
        \sctikzfig{figs/PIP_in_LSet}
	\end{minipage}
	\caption{$\Bval$ (left) and $\Bpol$ (right) in the free autonomisation of $\Set$ \cite{delpeuch_free_automization} (see also remark~\ref{rem:Autonomisation}).}
	\label{fig:VIP_PIP_in_Set}
\end{figure}

\begin{remark}\label{rem:Autonomisation}
	The free autonomisation of $(\Set,\times,1)$, denoted $(L\Set,\otimes,I)$, consists of formal string diagrams annotated with functions in $\Set$ \cite{delpeuch_free_automization}.
	The autonomous structure for the monoidal product $\otimes$ and the cartesian structure for the cartesian product $\times$ being different prevents $L\Set$ from having all objects be isomorphic to the monoidal unit.
	Forward morphisms are embeddings via the strong monoidal fully faithful functor $F:\Set\to L\Set$.
	Cups $\varepsilon:A\otimes A^*\to I$ and caps $\eta:I\to A^*\otimes A$ are freely added, and allow to draw dual objects $A^*$ as backward wires.
	The diagrammatic notation for $L\Set$ in Figure~\ref{fig:VIP_PIP_in_Set} uses the isomorphism $\lambda:A^*\otimes B \cong B^A$ and the cap $\eta:I\to A^*\otimes A$. Morphisms in $\Set$ like $\mathrm{ev}:B^A\times A \to B$ embed again via $F$.
\end{remark}

\subsection{Parametric Bellman operators}

Since the category $\Lens = \Optic (\Set)$ is enriched in $\Set$, we can form its category $\Para_\Set (\Lens)$ of externally parametrised morphisms. A morphism $\binom{X}{X'} \to \binom{Y}{Y'}$ of this category consists of a parameter set $P$, a forwards pass function $P \times X \to Y$ and a backwards pass function $P \times X \times Y' \to X'$.

The simplest Bellman operator for RL, the one for SARSA, is a morphism in this category of type $\B : \binom{1}{S \times A \times \R} \to \binom{S \times A}{\R}$,
with parameter set $\ret = S \times A \times \R \times S \times A$, where the backward pass function is $r, v \mapsto r + \gamma v$ (figure~\ref{fig:SARSA_lens}).

\begin{wrapfigure}{R}{0.35\textwidth}
   \sctikzfig{figs/SARSA_lens2}
	\caption{Target computation in SARSA as a parametrised lens, and its $\K$-image in $\Set$.}
	\label{fig:SARSA_lens}
\end{wrapfigure}

We lift the functor $\K : \Lens^\op \to \Set$ to the functor $\Para_\Set (\K) : \Para_\Set (\Lens^\op) \to \Para_\Set (\Set)$.
Applying this functor to $\B$ results in a function
\begin{align*}
	\Para_\Set (\K) (\B) : \ret \times \R^{S \times A} &\to S \times A \times \R \\
	((s, a, r, s', a'), Q) &\mapsto (s, a, r + \gamma Q (s', a'))
\end{align*}

Here we very informally think of $S \times A \times \R$ as a `discrete cotangent vector' at $Q \in \R^{S \times A}$. Slightly more precisely, the actual cotangent space is $\R^{S \times A}$, with $S \times A \times \R$ representing a scaled basis vector via the embedding $S \times A \times \R \hookrightarrow \R^{S \times A}$,
$(s, a, r) \mapsto (s', a') \mapsto \begin{cases}
	r &\text{ if } s = s', a = a' \\
	0 &\text{ otherwise}
\end{cases} $.
This differential geometry perspective is heavily inspired by Myers' categorical systems theory \cite{david-jaz-book} and related unpublished work in progress of Capucci, and we leave it for future work to make its application to reinforcement learning precise.

We do not provide a general definition of Bellman operators, and consider this a representative `definition by example'. In general, Bellman operators will be optics from value functions to deltas of value functions, parametrised by a \emph{sample} which is data received from the environment.

%

In methods where the sample requires the operator to make use of the continuation only once, this parametric operator can be represented as a lens.
SARSA 
is an example of this, where the usage of $Q$ by the target is \emph{linear} in the sense of linear type theory.
Setting aside the convergence properties of the Bellman operator, we treat it from this point on as a morphism $G:\ret\times \R^S\to \R^S$ in $\Set$.
This morphism becomes a central part of our formalization of an RL model, explained next.


\section{Models, agents and environments}

A model for an RL method contains the data to generate the policy from certain inner parameters and the data to update those parameters based on bootstrapping and/or samples.
It matches the structure of a lens from model parameters to agent interface, which we annotate in figure~\ref{fig:RL_models}(a).

The forward map uses parameters of the method to generate a policy for the agent's interaction with the environment.
In Q-learning for example, which is a TD method (figure~\ref{fig:RL_models}(b)), $P:\R^{S\times A}\to (\mathcal{P}A)^S$ takes the current Q-table $Q:\R^{S\times A}$ and returns the greedy policy $\pi:(\mathcal{P}A)^S$ defined by $\pi(s)=\argmax_a Q(s,a)$.

The backward map takes the return from the agent and the current parameters to generate an update target as a (often discrete) cotangent vector
for the parameters.
The return usually takes the form of some product of types $S$ (states), $A$ (actions) and $\R$ (rewards).
In SARSA (figure~\ref{fig:RL_models}(b)), $G$ takes a sample $(s,a,r,s',a')$ (right input) and the bootstrapped Q-table (upper input) to calculate the target $r+\gamma Q(s',a')$, which is the direction of the cotangent vector at $(s,a)$.
In DP (figure~\ref{fig:RL_models}(d)), there is no return from the agent, and the update target is the output of the Bellman operator $\Bval$ as a section of the cotangent bundle: for every state $s\in S$, $\Bval(V)(s)$ defines the direction that $V (s)$ must change to.

Certain DP methods like GPI or asynchronous DP \cite{parallel_computation} benefit from treating the two Bellman operators as separate backward morphisms (figure~\ref{fig:RL_models}(e)).

This also captures some deep reinforcement learning methods, which use a gradient-based approximation of the value functions instead of a tabular approach.
An example is actor-critic methods \cite{actor_critic}.
We will not give a complete motivation of their definition, but in short, the traditional policy is a neural network called the \emph{actor}, and is complemented by an additional function called the \emph{critic} that learns a baseline value function. 
The actor $P:\Theta\to (S\to DA)$ and the critic $V:\Omega\to (S\to \R)$ appear in the forward map of the diagram (figure~\ref{fig:RL_models}(f)), but only the actor outputs to the right interface, as the critic does not act on the environment.
Being a deep RL method, the actor map $P:\theta\mapsto \pi_\theta$ is a neural network, and the associated backward loss map $\mathcal{L}_\mathrm{ac}:\Theta\times\Omega\times SA\R S\to \Delta\Theta$ is defined by the improvement of expected return $(\theta,\omega,s,a,r,\_)\mapsto (r-V_\omega(s))\nabla_\theta\log\pi_\theta(s,a)$ with the baseline given by $V_\omega$.
The critic map $V:\omega\mapsto V_\omega$ has as the backward loss map $\mathcal{L}_\mathrm{cr}:\Omega\times SA\R S\to\Delta\Omega$ the reduction of policy update variance $(\omega,s,a,r,s')\mapsto r + \gamma V_\omega(s')-V_\omega(s)$.

\begin{figure}[ht]
	\begin{minipage}{0.33\textwidth}
	    \sctikzfig{figs/generic_model}
	\end{minipage}\begin{minipage}{0.33\textwidth}
	    \sctikzfig{figs/TD_model}
	\end{minipage}\begin{minipage}{0.33\textwidth}
		\sctikzfig{figs/MC_model}
	\end{minipage} \\
	\begin{minipage}{0.33\textwidth}
	    \sctikzfig{figs/DP_model}
	\end{minipage}\begin{minipage}{0.33\textwidth}
	    \sctikzfig{figs/ADP_model}
	\end{minipage}\begin{minipage}{0.33\textwidth}
		\sctikzfig{figs/AC_model}
	\end{minipage}
	\caption{Lenses for $(a)$ a generic RL model, $(b)$ TD, $(c)$ MC, $(d)$ DP, $(e)$ GPI and $(f)$ Actor-Critic methods. The drawing of the backward map as a morphism with an input from the top is merely a stylistic choice, where it should be understood as a morphism with two inputs, the bootstrap and the sample.}
	\label{fig:RL_models}
\end{figure}

This generic model lens, and any of the variants above, embed into $\Optic(\C)$ which is extended to $\Optic^\I(\C)$ by an iteration functor $\I$ defined next.
The left interface to the model optic is closed by two pieces of data: An initial state $q_0:I\to M\otimes \Theta$ and an update rule $i:M\otimes \Delta\Theta\to M\otimes \Theta$ that acts as the iterator.
The bootstrapping type $M$ is usually the whole parameter space $\Theta$, but the optic iteration functor allows $M$ to be any state space, e.g. a subset of $\Theta$ or an alternative encoding of it.
In gradient-free methods, the update is generally pointwise as in \eqref{eq:model_update}.
Conversely, gradient-based methods use neural network optimizers like stochastic gradient descent, Adam and other variations as the update rule \cite{foundations-gradient-learning}.


The right interface of a model, which is a morphism $\binom{S}{I} \to \binom{A}{F}$ in $\Para(\Optic(\C))$, parametrises an agent. This parametrised morphism will itself interact with an environment that is an iteration context.
The coend in the environment is taken over states of the Markov chain. 

Offline methods, unlike online methods, interact with the agent only in a trivial way by showing it experiential samples. This is shown by the types $M=S\times A\times F$ and $M'=I$, by which the continuation ignores the agent's action and just projects the action and feedback as a response to the agent.
Moreover, the iterator type $\C(M'\otimes I, M \otimes S)\cong\C(I, M \otimes S)$ coincides with the initial state, which reflects the fact that the environment samples experiences $(s,a,f)$ from a distribution defined by a dataset (figure~\ref{fig:environments}(a,b)).

\begin{figure}[ht]
	\begin{minipage}{0.7\textwidth}
    \sctikzfig[.75]{figs/onpolicy_SARSA}
	\end{minipage}\begin{minipage}{0.3\textwidth}
    \sctikzfig[.75]{figs/offpolicy_Q-learning}
	\end{minipage}
	\caption{SARSA is on-policy (left two). Q-learning is off-policy (right).}
	\label{fig:onpolicy_offpolicy}
\end{figure}

To clarify the interplay between these the three structures described in this section, we look at the role played by internal and external policies in on- and off-policy methods.
First, figure~\ref{fig:onpolicy_offpolicy}(left) shows the full representation of SARSA. It consists of a model optic parametrising two copies of an agent that are composed with a 2-hole environment. The policy evaluated by both instances is the same, and the return to the model consists of $(s,a,r)$ from the first agent optic and $(s',a')$ from the second. SARSA is an on-policy method, as the policy deployed to obtain $a'\sim \pi(s')$ is the same as the one used to compute the first action $a\sim \pi(s)$. Calculating the target $G$ from the sample $(s,a,r,s',a')$ is equivalent to calculating $G$ from $(s,a,r,s')$ and its internal policy $\pi$, even though the model does not know the environment's dynamics $k$.
This is why the same method can be equivalently specified by the middle diagram.

On the other hand, Q-learning (figure ~\ref{fig:onpolicy_offpolicy}(right)) is an off-policy method, because the last action is computed by an internal policy $\pi'=\argmax$ different from the one being deployed.

\begin{figure}[ht]
	\begin{minipage}{0.25\textwidth}
        \sctikzfig{figs/online_environment}
	\end{minipage}\begin{minipage}{0.30\textwidth}
        \sctikzfig{figs/offline_environment}
	\end{minipage}\begin{minipage}{0.25\textwidth}
        \sctikzfig{figs/contextual_bandit_environment}
	\end{minipage}\begin{minipage}{0.20\textwidth}
        \sctikzfig{figs/multiarmed_bandit_environment}
	\end{minipage}
	\caption{Online (a) and offline (b) RL environments. Contextual (c) and multi-armed (d) bandit environments. Omitted arrows are the unit. The offline continuation is a projection $p$ of $A\times F$.}
	\label{fig:environments}
\end{figure}

\subsection{Prediction and bandit problems} \label{sec:prediction_bandit}

The presented framework handles RL \textbf{prediction} problems for free in all the previous methods by trivialising the set $A=1$, which pinpoints the idea that \emph{a prediction algorithm is a control algorithm where there's no choice of actions}.
For example, MC prediction of the long-term value of states from $n$-long episodes becomes an optic $\binom{\R^S}{T^*_sS}\to\binom{I}{S\R^n}$, and $1$-TD prediction becomes $\binom{\R^S}{T^*_sS}\to\binom{I}{S\R S}$. The forward maps for both are trivial since the agent has no policy to execute, perhaps better called \textbf{observer} rather than agent here.
The corresponding environments have the type of a MRP.

Moreover, \textbf{bandit problems} emerge by trivialising $M'=I$ (figure~\ref{fig:environments}(c,d)).
In particular, contextual bandits involve finding the best action in $A$ associated to a particular state in $M$ for which only partial information of type $S$ is given, yielding feedback in $F$.
This action does not affect further distributions of states, so the object between the continuation and the update rule is trivial.
Multi-armed bandit problems are a further special case, characterized by environments  whose only non-trivial morphism is the continuation $k:A\to F$.

\bibliographystyle{eptcsalpha}
\bibliography{rl-in-cybcat}

\appendix
\section{Appendix}\label{appendix}

\begin{proof}[Proof (\ref{prop:well-defined_iterator})]
	We have to prove that the following function is well-defined:
	\begin{equation}
		\Optic\left(\binom{X}{X'},\binom{Y}{Y'}\right) \to \left[ \I\binom{X}{X'} \to \I\binom{Y}{Y'} \right] \label{eq:it_well_def}
	\end{equation}
	Following Riley's proof method of sequential composition of optics \cite{riley_optics}, the uncurried form of the above function has as domain:
	\begin{align*}
		&\left(\int^M \C(X,M\otimes Y)\times \C(M\otimes Y',X')\right) \times \left(\int^N \C(I,N\otimes X)\times\C(N\otimes X',N\otimes X)\right) \\
		\cong &\int^{M,N} \C(X,M\otimes Y)\times \C(M\otimes Y',X') \times \C(I,N\otimes X)\times\C(N\otimes X',N\otimes X) \tag{coend-Fubini} \\
	\end{align*}
	By the universal property of coends it suffices to construct maps natural in $M$ and $N$ into the codomain of \eqref{eq:it_well_def}:
	\begin{align*}
		& \C(X,M\otimes Y)\times \C(M\otimes Y',X') \times \C(I,N\otimes X)\times\C(N\otimes X',N\otimes X) \\
		\to & \C(I,N\otimes M\otimes Y)\times \C(N\otimes M\otimes Y',N\otimes M\otimes Y) \tag{composition in $\C$} \\
		\to &\int^P \C(I,P\otimes X)\times\C(P\otimes X',P\otimes X) \tag{$\mathrm{copr}_{N\otimes M}$}
	\end{align*}
	where the first map takes morphisms $f,f',x_0,i$ to
	\begin{align*}
		& I \xrightarrow{x_0} N\otimes X \xrightarrow{N\otimes f} N\otimes M\otimes Y \\
		& N\otimes M\otimes Y' \xrightarrow{N\otimes f'} N\otimes X \xrightarrow{i} N\otimes X \xrightarrow{N\otimes f} N\otimes M\otimes Y
	\end{align*}
	This also admits a graphical representation, depicted in Figure~\ref{fig:well-defined_iterator}.
\end{proof}

\begin{figure}[ht]
	\sctikzfig{figs/well-defined_iterator}
	\caption{Composition of an optic with an iterator yields another iterator.}
	\label{fig:well-defined_iterator}
\end{figure}

\begin{proof}[Proof (\ref{prop:functorial_iterator})]
	Let $f = (N, f, f') : \binom{X}{X'} \to \binom{Y}{Y'}$ and $g = (P, g, g') : \binom{Y}{Y'} \to \binom{Z}{Z'}$ be two morphisms in $\Optic(\C)$. Preservation of identity is shown by:
	\[ \I(I,1_X,1_{X'}):(M,x_0,i)\mapsto (M\otimes I,x_0;(I\otimes 1_X),(M\otimes 1_{X'});i;(M\otimes 1_X)) \]
	Preservation of composition is shown by the isomorphic images of $\I(N,f,f');\I(P,g,g')$, which maps $(M,x_0,i):\I \binom{X}{X'}$ to the state space $M\otimes N\otimes P$, the initial state $I \overset{x_0}\longrightarrow M \otimes X \xrightarrow{M \otimes f} M \otimes N \otimes Y\xrightarrow{M\otimes N\otimes g} M\otimes N\otimes P\otimes Z$  and the iterator
	\[ M\otimes N\otimes P \otimes Z' \xrightarrow{M\otimes N\otimes g'} M \otimes N \otimes Y' \xrightarrow{M \otimes f'} M \otimes X' \overset{i}\longrightarrow M \otimes X \xrightarrow{M \otimes f} M \otimes N \otimes Y \xrightarrow{M\otimes N\otimes g} M\otimes N\otimes P\otimes Z \]
	which defines the element in $\I \binom{Z}{Z'}$, and $\I(N\otimes P,(f;N\otimes g),(N\otimes g';f'))$, which maps $(M,x_0,i)$ to the same state space, the initial state $x_0;(M\otimes (f;N\otimes g))$, and the iterator $(M\otimes (N\otimes g');f');i;(M\otimes (f;N\otimes g))$.
\end{proof}

\begin{figure}[ht]
	\sctikzfig{figs/iterated_optic}
	\caption{Typical morphism in $\Optic^\I (\C) = \pi_0^* \left( \Para^\I (\Optic (\C)) \right)$ which consists of a morphism in $\Optic (\C)$ extended by a representative element $(M,x_0,i)\in \I \binom{X}{X}$. Adapted from \cite{iteration_optics}.}
	\label{fig:iterated_optic}
\end{figure}

\begin{proof}[Proof (\ref{prop:well-defined_streams})]
	Considering that $\Optic(\Set)\cong\Lens$, the domain and codomain functors of the transformation are $\K\times\I:\Lens^\op\times \Lens\to\Set$ and $\V^\omega:\Lens\to\Set$, where $\V^\omega$ is the forwards pass functor $\V:\Lens\to\Set$ followed by the stream functor $(-)^\omega:\Set\to\Set$.
	This being a purely covariant functor makes the dinaturality condition into the following pentagon identity for every lens $\lambda:{X\choose X'}\to{Y\choose Y'}$ whose forward and backward maps we denote $f$ and $f'$:
	\[\begin{tikzcd}[ampersand replacement=\&,cramped]
		\&\& {\mathbb{K}{X\choose X'}\times\mathbb{I}{X\choose X'}} \&\& {\mathbb{V}^\omega{X\choose X'}=X^\omega} \\
		{\mathbb{K}{Y\choose Y'}\times\mathbb{I}{X\choose X'}} \\
		\&\& {\mathbb{K}{Y\choose Y'}\times\mathbb{I}{Y\choose Y'}} \&\& {\mathbb{V}^\omega{Y\choose Y'}=Y^\omega}
		\arrow["{\langle-\mid-\rangle_{X\choose X'}}", from=1-3, to=1-5]
		\arrow["{\mathbb{V}^\omega(\lambda)=f^\omega}", from=1-5, to=3-5]
		\arrow["{\mathbb{K}(\lambda)\times\mathbb{I}{X\choose X'}}", from=2-1, to=1-3]
		\arrow["{\mathbb{K}{Y\choose Y'}\times\mathbb{I}(\lambda)}"', from=2-1, to=3-3]
		\arrow["{\langle -\mid -\rangle_{Y\choose Y'}}", from=3-3, to=3-5]
	\end{tikzcd}\]
	For $k,(M,(m_0,x_0),i)$ in $\K{Y\choose Y'}\times\I{X\choose X'}$, the diagram commutes when the streams $f^\omega \braket{\lambda;k\mid M,(m_0,x_0),i}$ and $\braket{k\mid M,(m_0,f(x_0)),j}$ are equal, where $j=(M\times f');i;(M\times f)$.
	We proceed by coinduction, showing that the streams have equal heads and tails; an equivalent perspective is that they are generated by bisimilar (in fact isomorphic) state machines.
	The heads of both streams is $f(x_0)$, and their tails operate as state machines for states $(m_0,x_0)$:
	\begin{align*}
		f^\omega\braket{\lambda;k\mid M, i(m_0,(\lambda;k)(x_0)),i} &= f(i(m_0,(\lambda;k)(x_0))_1) :: \braket{\cdots} \\
		\braket{k\mid M, j(m_0,k(f(x_0))), j} &= j(m_0,k(f(x_0)))_1 :: \braket{\cdots}
	\end{align*}
	Thus we show for all pairs $(m,x)$,
	\begin{align*}
		& f(i(m,(\lambda;k)(x))_1) \\
		= & f(i(m,(f;k;f')(x))_1) \tag{Composition of lens with continuation} \\
		= & ((M\times f');i;(M\times f))(m,k(f(x)))_1 \\
		= & j(m,k(f(x)))_1
	\end{align*}
\end{proof}


\end{document}

%% file: cybercatrl.tikzstyles
\tikzstyle{circ}=[fill=white, draw=black, shape=circle]
\tikzstyle{blank}=[fill=white, draw=white, shape=circle]
\tikzstyle{copy}=[fill=white, draw=black, shape=circle, minimum height=0.2cm, inner sep=0]
\tikzstyle{varCopy}=[fill=black, draw=black, shape=circle, minimum height=0.15cm, inner sep=0]
\tikzstyle{1morph}=[fill=white, draw=black, shape=rectangle, minimum width=0.5cm, minimum height=0.5cm, inner sep=0.1cm]
\tikzstyle{2morph2}=[fill=white, draw=black, shape=rectangle, minimum width=1cm, minimum height=2cm]
\tikzstyle{2morph}=[fill=white, draw=black, shape=rectangle, minimum width=0.6cm, minimum height=1.1cm, inner sep=0.1cm]
\tikzstyle{nmorph}=[fill=white, draw=black, shape=rectangle, minimum height=6cm, minimum width=1cm, inner sep=0.1cm]
\tikzstyle{1state}=[fill=white, draw=black, regular polygon, regular polygon sides=3, minimum height=0.5cm, regular polygon rotate=-30]
\tikzstyle{dbox}=[fill=white, draw=black, dashed, shape=rectangle, minimum width=2cm, minimum height=1cm, inner sep=0.1cm]
\tikzstyle{vdbox}=[fill=white, draw=black, dashed, shape=rectangle, minimum width=2cm, minimum height=1.5cm, inner sep=0.1cm]
\tikzstyle{bigbox}=[fill=white, draw=black, dashed, shape=rectangle, minimum width=2cm, minimum height=4cm, inner sep=0.1cm]
\tikzstyle{2state}=[inner sep=0.05cm, fill=white, draw=black, isosceles triangle, minimum width=1.25cm, isosceles triangle apex angle=90, shape border rotate=180]
\tikzstyle{var2state}=[inner sep=0.05cm, fill=white, draw=black, isosceles triangle, minimum width=1.25cm, isosceles triangle apex angle=60, shape border rotate=180]
\tikzstyle{g2state}=[inner sep=0.05cm, fill=white, draw=black, isosceles triangle, minimum width=6cm, isosceles triangle apex angle=110, shape border rotate=180]
\tikzstyle{bigstate}=[inner sep=0.05cm, fill=white, draw=black, isosceles triangle, minimum width=3cm, isosceles triangle apex angle=110, shape border rotate=180]
\tikzstyle{bigeffect}=[inner sep=0.05cm, fill=white, draw=black, isosceles triangle, minimum width=3cm, isosceles triangle apex angle=110]
\tikzstyle{g2effect}=[inner sep=0.05cm, fill=white, draw=black, isosceles triangle, minimum width=6cm, isosceles triangle apex angle=110]
\tikzstyle{2effect}=[inner sep=0.05cm, fill=white, draw=black, isosceles triangle, minimum width=1.25cm, isosceles triangle apex angle=90]
\tikzstyle{b2effect}=[inner sep=0.05cm, fill=white, draw=black, isosceles triangle, minimum width=2cm, isosceles triangle apex angle=90]
\tikzstyle{Medium box}=[fill=white, draw=black, shape=rectangle, tikzit shape=rectangle, minimum width=1.5cm, minimum height=1.5cm]
\tikzstyle{eval}=[fill=white, draw=black, shape=circle]
\tikzstyle{textrot90}=[rotate=90]

\tikzstyle{midArrow}=[-, decoration={{markings,mark=at position .5 with {\arrow{>}}}}, postaction=decorate]
\tikzstyle{Gray line}=[-, draw={rgb,255: red,191; green,191; blue,191}, line width=0.8]
\tikzstyle{arrow}=[->]
\tikzstyle{line}=[-]
\tikzstyle{grayfill}=[-, fill={rgb,255: red,220; green,220; blue,220}, tikzit fill={rgb,255: red,220; green,220; blue,220}]
\tikzstyle{whitefill}=[-, fill=white, tikzit fill=white]
\tikzstyle{dashed}=[-, dash pattern=on 3pt off 3pt]